\pdfoutput=1
\documentclass[11pt,a4paper]{article}
\usepackage[utf8]{inputenc}
\def\final{1}
\usepackage{color}
\usepackage{subfigure}
\usepackage{url}
\usepackage{graphicx}
\usepackage{amsmath}
\usepackage{amssymb}
\usepackage{amsthm}
\usepackage{bm}
\usepackage[plainpages=false]{hyperref}

\usepackage{titlesec}
\titleformat{\section}
{\normalfont\Large\bfseries\filcenter}{\thesection.}{1 ex}{}
\titleformat{\subsection}
{\normalfont\normalsize\bfseries}{\thesubsection.}{1 ex}{}
\titleformat{\subsubsection}[runin]
{\normalfont\normalsize\bfseries\filcenter}{\thesubsection.}{1 ex}{}

\usepackage[charter]{mathdesign}
\usepackage[mathcal]{eucal}

\usepackage[margin=1.2 in]{geometry}

\usepackage[square,numbers,sort&compress]{natbib}

\ifnum\final=0
\newcommand{\mynote}[1]{\marginpar{\tiny\sf #1}}
\else
\newcommand{\mynote}[1]{}
\fi

\usepackage{thmtools,thm-restate}
\declaretheoremstyle[qed=$\diamond$,headpunct={ --- },headfont=\normalfont\itshape]{myremark}
\declaretheoremstyle[qed=$\blacksquare$,bodyfont=\normalfont]{mydefinition}
\declaretheorem[name=Theorem]{Thm}
\declaretheorem[within=section,name=Lemma]{Lem}
\declaretheorem[sibling=Lem,name=Proposition]{Prop}

\declaretheorem[style=myremark,sibling=Lem,name=Remark]{Rem}
\declaretheorem[style=mydefinition,sibling=Lem,name=Definition]{Def}

\renewcommand{\eqref}[1]{({\ref{eq:#1}})}

\newcommand{\varX}{\ensuremath{\mathcal{X}}}
\newcommand{\varY}{\ensuremath{Y}}
\newcommand{\varZ}{\ensuremath{Z}}
\newcommand{\CayMen}{\mathcal{C}}

\newcommand{\CayMenSurj}{\varphi}

\newcommand{\DetVSymSurj}{\phi}

\newcommand{\DetV}{\mathcal{M}}
\newcommand{\DetVSym}{\mathcal{M}_{sym}}
\newcommand{\DetVSymN}{\mathcal{M}_{sym}^\diamond}
\newcommand{\SphPr}{\nu_h}
\newcommand{\smp}{\ensuremath{\mathcal{S}}}
\newcommand{\srate}{\ensuremath{\rho}}

\newcommand{\Ker}{\operatorname{Ker}}

\newcommand{\codim}{\operatorname{codim}}
\newcommand{\rk}{\operatorname{rk}}

\newcommand{\id}{\operatorname{id}}

\newcommand{\spn}{\operatorname{span}}

\newcommand{\coh}{\operatorname{coh}}

\newcommand{\Sm}{\operatorname{Sm}}

\newcommand{\calA}{\mathcal{A}}

\newcommand{\calP}{\mathcal{P}}

\newcommand{\calX}{\mathcal{X}}

\newcommand{\RR}{\ensuremath{\mathbb{R}}}

\newcommand{\CC}{\ensuremath{\mathbb{C}}}

\newcommand{\KK}{\ensuremath{\mathbb{K}}}

\makeatletter
\providecommand*{\diff}%
{\@ifnextchar^{\DIfF}{\DIfF^{}}}
\def\DIfF^#1{%
\mathop{\mathrm{\mathstrut d}}%
\nolimits^{#1}\gobblespace
}
\def\gobblespace{%
\futurelet\diffarg\opspace}
\def\opspace{%
\let\DiffSpace\!%
\ifx\diffarg(%
\let\DiffSpace\relax
\else
\ifx\diffarg\[%
\let\DiffSpace\relax
\else
\ifx\diffarg\{%
\let\DiffSpace\relax
\fi\fi\fi\DiffSpace}
\makeatother

\begin{document}
\title{Coherence and Sufficient Sampling Densities for Reconstruction in Compressed Sensing}
\author{Franz J. Király\thanks{Department of Statistical Science, University College London. \url{f.kiraly@ucl.ac.uk}}
\and
Louis Theran\thanks{Inst. Math., FU-Berlin. \url{theran@math.fu-berlin.de}}}
\date{}

\maketitle

\begin{abstract}
We give a new, very general, formulation of the compressed sensing problem in terms of coordinate projections
of an analytic variety, and derive sufficient sampling rates for signal reconstruction.  Our bounds
are linear in the \emph{coherence} of the signal space, a geometric parameter independent of the
specific signal and measurement, and logarithmic in the ambient dimension where the signal is presented.
We exemplify our approach by deriving sufficient sampling densities for low-rank matrix completion and distance
matrix completion which are independent of the true matrix.
\end{abstract}

\section{Introduction}\label{Sec:intro}
\subsection{Compressed Sensing, Randomness, and $n\log n$}\label{sec:intro.Nyquist}
Compressed sensing is the task of recovering a \emph{signal} $x$ from some low-complexity measurement(s) $\smp (x)$, the \emph{samples} of $x$. The sampling process, that is, the acquisition process of the sample $\smp (x)$, is usually random and undirected, and it comes with a so-called \emph{sampling rate}. Increasing the sampling rate usually improves quality of the reconstruction but comes at a cost, whereas decreasing it makes the acquisition easier but hinders reconstruction. Therefore, a central question of compressed sensing is what the minimal sampling rate has to be, in order to allow reconstruction of the signal from the sample.

The oldest and best-known example of this is the Nyquist sampling theorem~\cite{Nyquist}, which roughly states that a signal
bandlimited to frequency $f$ has to be sampled with a frequency/density at least $2f$, in order to allow reconstruction.
Landau's \cite[Theorem 1]{Landau} and a simple Poisson approximation (or coupon collector) argument imply that for
uniform sampling, to ascertain a \emph{density} of $n$, a rate $\rho = \Omega\left(\frac{f}{n}\log n\right)$, or
a number of $\Omega(f\log n)$ total samples, is necessary and sufficient. The ratio $f/n$ can be interpreted as the average
informativity of one equidistant (non-random) measurement, in the sense of how much it contributes to reconstruction.

Sufficient sampling rates of the form ``$\Omega(\frac{k}{n}\log n)$'', with $k$ some problem-specific constant and $n$
a natural measure of the problem size,
appear through the modern compressed sensing literature.  A
few examples are:
image reconstruction \cite[Theorem~1]{Donoho},
matrix completion \cite[Theorems~1.1, 1.2]{CR09},
dictionary learning \cite[Theorems~7,8]{Spielman},
and phase retrieval, \cite[Theorem~1.1]{Candes:uq}.
Usually, these bounds are derived by analyzing some optimization problem, or information
theoretic thresholds, under assumptions, which, while not overly restrictive, are
very specific to one problem.

We argue that those bounds on the sampling rates are epiphenomena of guiding principles in compressed sensing,
similar to the Nyquist sampling bound.  To do this, we give a general formulation of the problem in
which, associated to the sampling process $\smp(x)$ there are two numerical invariants: the \emph{coherence}
$\coh(\smp)$ and the \emph{ambient dimension} $n$.  The ``dictionary'' between the classical setting and
our novel framework for compressed sensing is, intuitively:\\
\begin{center}
\begin{tabular}{|l|l|}\hline
\textbf{Classical}  & \textbf{Compressed Sensing} \\ \hline
Signal Space & Signal Manifold $\varX$ \\ \hline
Sampling (random) & Random Projection $\smp$ \\ \hline
Bandlimit $f$ & Manifold Dimension $\dim (\varX)$ \\ \hline
Sampling Density $n$ & Ambient Dimension $n$ \\ \hline
Informativity $f/n$ & Coherence $\coh (\varX)$ \\
& (in general $\neq \dim (\varX) / n$ )\\ \hline
Sampling Rate $\rho$ & Sampling Probability $\srate$\\ \hline
\end{tabular}
\end{center}

Our main Theorem \ref{Thm:cohp} says that a sampling
rate of $\Omega(\coh(\smp)\log n)$ is sufficient for signal reconstruction, w.h.p.
Further, we will see that $\dim(\smp)/n\le \coh(\smp)\le 1$, where $\dim(\smp)$ is the number of degrees of
freedom in choosing the signal $x$.  This relation shows that, when the coherence is near the lower bound,
$\dim(\smp)$ is in complete analogy with the bandlimit $f$ from the classical setting and that
$O(\dim(\smp)\log n)$ independently chosen measurements are sufficient for signal reconstruciton.
Coherence captures the \emph{structural} constraint on the sufficient sampling
rate, and the $\log n$ term appears because measurements are chosen
independently, with the same probability.  This result is existentially optimal, since the $\log n$
terms cannot be removed in some examples.

\subsection{The Mathematical Sampling Model}\label{Sec:intro.sampling}
We will consider the following sampling model for compressed sensing:
the signals $x$ will be considered as being contained in $x\in\KK^n$, with
the standard basis of elementary vectors.  The field $\KK$ is always $\RR$ or $\CC$
in this paper.

This setup imposes no restriction on the signal $x$, since we are only fixing a
finite/discrete representation by $n$ numbers, and the continuous case is recovered
by taking the limit in $n$.  Examples include representing $x$ as a bandlimited
DFT or as a finite matrix instead of a kernel function and graphon.

We will model the sampling process by a map $\smp: \KK^n \rightarrow \KK^m$, which chosen uniformly from
a restricted family. For example, in the case of the bandlimited signal, the mapping $\smp$ would be initially
linear, of the form $x(t_j)=\sum_i a_i \varphi(t_j)$ with $t_j$ being the chosen sampling points, $\varphi$ the
Fourier basis, and $a_i$ the Fourier coefficients of $x$. The map would send the $a_i$ to the $x(t_j)$. To obtain
a universal formulation, we now perform a change of parameterization on the left side, by changing it to contain
\emph{all possible measurements}, instead of the signal. In the example, the signal $x(t)$ would be parameterized
not by the Fourier coefficients $a_i$, but instead by \emph{all possible} $x(t_j)$, being a much larger set than
the actual measurements $x(t_j)$ contained in $\smp(x)$ when sampling once.

This re-parameterization makes $n$ large, but it makes the single coordinates dependent as well.
(In the example, the dependencies are linear.)  In other words, under the re-parameterization
\emph{all possible signals} lie in a \emph{low-dimensional submanifold} $\varX$ of $\KK^n$.
The sampling process $\smp$ then becomes a \emph{coordinate projection map}
$\smp : \varX \to\KK^m$ onto $m$ entries of the true signal $x\in \calX$ chosen uniformly
and independently.

Under the re-parameterization, $\smp$ is chosen independently of the problem.  All the structural
information about the signal $x$ is moved to the manifold $\calX$, which determines dependencies between
the coordinates.  The key concept of \emph{coherence} will then be a property of $\calX$, as opposed
to the usual view where compression and sampling constraints are assumed the particular signal $x$ or
enforced by special assumption on the sampling operator $\smp$.

\subsection{Contributions}\label{Sec:intro.contribs}
Our main contributions, discussed in more detail below, are:
\begin{itemize}
\item A problem-independent formulation of compressed sensing.
\item A problem-independent generalization of the sampling density, given by the coherence $\coh(\varX)$ of the signal class $\varX$. Determination of coherence for linear sampling, matrix completion, combinatorial rigidity and kernel matrices.
\item Derivation of problem-independent bounds for the sampling rate, taking the form $\Omega \left(\coh(\varX)\cdot n\log n\right)$. We recover bounds known in compressed sensing literature, and derive novel bounds for combinatorial rigidity and kernel matrices.
\item Explanation of the $\log n$ term as an epiphenomenon of sampling randomness.
\end{itemize}
\subsection{Main theorem: coherence and reconstruction}
Our main result elates the coherence of the signal space $\varX$ to the sampling rate $\srate$ of a typical
signal $x\in\varX,$ which suffices to achieve reconstruction of $x$. We show:
\begin{restatable}{Thm}{cohpthm}\label{Thm:cohp}
Let $\varX\subseteq\mathbb{K}^n$ be an irreducible algebraic variety, let $\Omega$ be the
projection onto a set of coordinates, chosen independently with probability $\srate$, let $x\in \varX$ be
generic. There is an absolute constant $C$ such that if
$$\srate \ge C\cdot \lambda\cdot \coh (\varX)\cdot  \log n,\quad\mbox{with}\;\lambda\ge 1,$$
then $x$ is reconstructible from $\Omega(x)$ - i.e., $\Omega^{-1}(\Omega (x))$ is finite -
with probability at least
$1-3n^{-\lambda}$.
\end{restatable}
Here \emph{generic} can be taken to mean that if $x$ is sampled from a (Hausdorff-)continuous probability
density on $\varX$, then the statement holds with probability one.

\subsection{Applications}
We illustrate Theorem \ref{Thm:cohp} by a number of examples, which will also show that the
bounds on the sampling rate there cannot be lowered much.
\paragraph{Linear Sampling and the Nyquist bound}
When the sampling manifold $\varX$ is a $k$-dimensional linear subspace of $\KK^n$,
as in the case of the bandlimited signal,
Proposition \ref{Prop:cohbnds}, below, implies that $\frac{k}{n}\le \coh (\varX) \le 1$.
We then recover a statement which is qualitatively similar to the random version of the Nyquist bound:
\begin{Thm}\label{Thm:linear}
Let $\varX\subseteq \KK^n$ be a linear space. Let $x\in\varX$ be generic.
There is an absolute constant $C$, such that if each coordinate of $x$ is observed independently with probability
$$\srate \ge C\cdot \lambda \cdot \coh(\varX)\cdot \log n,\quad\mbox{with}\;\lambda\ge 1, $$
then, $x$ can be reconstructed from the observations with probability at least $1-3n^{-\lambda}$.
\end{Thm}
If $\calX$ is random in the sense of Definition~\ref{Def:maxinc}, then $\coh(\calX) = \frac{k}{n}$,
and the required number of samples is $O(k\log n)$, which is in line with our discussion in
section~\ref{sec:intro.Nyquist} regarding the Nyquist criterion.  In section \ref{sec:intro.fixed},
we will give a simple example showing that this cannot be improved.

\paragraph{Low-Rank Matrices}
Another important application is low-rank matrix completion. Here, $\varX$ is the set of low-rank $m\times n$ matrices of rank
$r$, which we show has  $\coh (\varX)= (mn)^{-1}\cdot r(m+n-r)$. We therefore obtain:

\begin{restatable}{Thm}{lowrankthm}\label{Thm:generic-low-rank}
Let $r\in \mathbb{N}$ be fixed, let $A$ be a generic $(m\times n)$ matrix of rank at most $r$.
There is an absolute constant $C$, such that if each entry of $A$ is observed independently with probability
$$\srate \ge C\cdot \lambda \cdot (mn)^{-1}\cdot r(m+n-r)\cdot \log mn, \quad\mbox{with}\;\lambda\ge 1,$$
then $A$ can be reconstructed from the observations with probability at least $1-3(mn)^{-\lambda}$.
\end{restatable}
Bounds of this type have been observed in \cite{CT10}, \cite{KMO10} and \cite{KTTU12}, while all of these results are
stated in the context of some reconstruction method, and therefore make sampling assumptions on the matrix $A$.
Thus, the novelty of Theorem \ref{Thm:generic-low-rank} is that it applies to a full-measure subset of low-rank
matrices.  Also, it has been noted already in \cite{CT10}
that the order of the bound in $m,n$ cannot be improved.

The analogue to Theorem \ref{Thm:generic-low-rank} holds with $m=n$
if $A$ is symmetric. We will also show that similar types of bounds hold for kernel matrices.

\paragraph{Distance Matrices}
A further related topic is the complexity of distance matrices, in which either the signal, or the sampling, exhibits the dependencies of a distance matrix (also sometimes called similarity matrix). The best-known case is that of an Euclidean distance matrix is an $(n\times n)$ matrix $A$ such that $A_{ij} = \|p_i - p_j\|^2$  if for some set of points $p_1,p_2,\ldots, p_n\in \mathbb{R}^d.$ The sampling rate in distance matrix completion describes (a) the density of random measurements needed to reconstruct and incomplete distance matrix, and, simultaneously, (b) describes the sampling threshold at which the points $p_i$ can be triangulated. On a theoretical side, the asymptotics of this phase transition has attracted a lot of attention in the context of combinatorial rigidity theory~\citep{KMT11}, where the exact bound for this phase transition has not been known except in the cases $d=1$ and $d=2$, i.e., points on the line and on the plane. By bounding the coherence of the set of distance matrices as $\coh(\varX) \le C \frac{d}{n}$ for some global constant $C$, we determine this sampling rate for all dimensions $r$:

\begin{restatable}{Thm}{rigiditythm}\label{Thm:rigidity}
Let $r\in \mathbb{N}$ be fixed, let $D$ be a generic distance matrix of $n$ points in $d$-space.
There is a global constant $C$, such that if each entry of $D$ is observed independently with probability
$$\srate \ge C\cdot \lambda \cdot d/n \cdot \log n, \quad\mbox{with}\;\lambda\ge 1,$$
then $D$ can be reconstructed from the observations with probability at least $1-3 n^{-\lambda}$.
\end{restatable}

In the language of rigidity theory~\citep{KMT11}, Theorem \ref{Thm:rigidity} says that with the
stated sampling rate $\srate$, the random graph $G_{n}(\srate)$ is generically rigid w.h.p. Because the minimum degree of a graph that is generically rigid in dimension $d$ must be at least $d$, the order of the lower bound on $\srate$ cannot be improved by more than a factor of $\log n$ - again,
this can be seen as a coupon collector's argument. Our result can be seen as a density extension of
Laman's Theorem~\citep{L70} to dimensions $d\ge 3$, which is known \cite{JSS07}
to imply a necessary and sufficient bound on the sampling density $\srate \ge n^{-1}(\log n + 2\log\log n + \omega(1))$
in dimension $2$. We will also argue that similar results hold for kernel distance matrices as well.

\subsection{Fixed coordinates and the logarithmic term}\label{sec:intro.fixed}
Before continuing, we want to highlight an important
conceptual point.  Since all the projections are linear, it could
be counter-intuitive that the number of measurements needed for reconstruction
is on the order of $\dim(\calX)\log n$ and not simply $\dim(\calX)$, especially
in light of the following (probably folklore) theorem, proven in the Appendix:
\begin{Thm}\label{Thm:randproj}
Let $\varX\subseteq\mathbb{K}^n$ be an algebraic variety of dimension $d$,
let $x\in\varX$. Let $\ell:\mathbb{K}^n\rightarrow\mathbb{K}^m$ be a generic linear map. If $m > d$,
then $x$ is uniquely determined by the values of $\ell(x)$, and the condition that $x\in\varX.$
\end{Thm}
Thus, we could guess naïvely that $\dim(\calX)$ total samples are enough.
The subtlety that the naïve guess misses is that we are dealing with \emph{coordinate} projections,
and that $\calX$ can be inherently aligned with the coordinate system in a way that
requires more samples.
Consider the case of a linear space $\calX$, as in Theorem \ref{Thm:linear} and
assume that $n/k$ is an integer. Let $\calX$ be spanned by $k$ vectors $b_1,\ldots, b_k$
which are supported on disjoint sets of $n/k$ coordinates and have $\sqrt{k/n}$ in the non-zero
coordinates.  It is easy to see that $\coh(\calX) = \frac{k}{n}$, which is minimal.  However,
to have any hope of reconstructing a point $x$, we need to measure at least one coordinate in the
support of each $b_i$. A coupon collector's argument then shows that, indeed, $\Omega(k\log k)$
samples are required, and this is $\Omega(k\log n)$ when $k = n^\epsilon$.  At the other
extreme, if $\calX$ is spanned by coordinate vectors $e_1,\ldots,e_k$, then $\Omega(n\log n)$
samples are needed.
Examples like these illustrate why coherence is the right concept: it depends on the
coordinate system chosen for $\KK^n$ and the embedding of $\calX$.  Dimension, on
the other hand, is intrinsic to $\calX$, so it can't capture the behavior of coordinate
projections.

\subsection{Acknowledgements}
FK is supported by Mathematisches Forschungsinstitut Oberwolfach (MFO),
and LT by the European Research Council under the European Union's Seventh Framework
Programme (FP7/2007-2013) / ERC grant agreement no 247029-SDModels.

\section{Coherence and Signal Reconstruction}\label{Sec:coherence}
\subsection{Coherence and Bounds on Coherence}\label{Sec:coherence.bounds}
In this section, we introduce our concepts, and define what the coherence should be. As discussed in section~\ref{Sec:intro.sampling}, the sampling process consists of randomly and independently observing coordinates of the signal without repetition, and this is no restriction of generality, as we have also discussed there.

\begin{Def}\label{Def:ransmp}
Let $\varX\subseteq \KK^{n}$ be an analytic variety. Fix coordinates $(X_1,\dots, X_{n})$ for $\KK^{n}.$ Let $\smp(\srate)$ be a the Bernouilli random experiment yielding a random subset of $\{X_1,\dots, X_{n}\}$ where each $X_i$ is contained in $\smp(\srate)$ independently with probability $\srate$ (the sampling density). We will call the projection map $\Omega : \varX\to \varY$ defined by $(x_1,\dots, x_n)\mapsto (\dots, x_i,\dots\;:\; X_i\in \smp (\srate))$ of $\varX$ onto the coordinates selected by $\smp (\srate)$, which is an analytic-map-valued random variable, a \emph{random sample} of $\varX$ with \emph{sampling rate} $\srate$.
\end{Def}
The coherence takes the place of the factor of oversampling needed to guarantee reconstruction.
Intuitively, it can be also interpreted as the infinitesimal randomness of a signal. We define it first for linear
sampling, as we have in the case of the bandlimited signal discussed in
section~\ref{Sec:intro.sampling}.  Figure \ref{fig:coh} (a) gives a schematic of the concept.

\begin{Def}\label{Def:coh}
Let $H\subseteq \KK^n$ be a $k$-dimensional affine space (for short, a $k$-flat).
Let $\calP:\KK^n\rightarrow H\subseteq \KK^n$ the unitary projection operator onto $H$, let $e_1,\dots, e_n$ be a
fixed orthonormal basis of $\KK^n.$ Then the {\it coherence} of $H$ with respect to the basis $e_1,\dots, e_n$ is defined as
$$\coh (H) = \max_{1\le i\le n} \|\calP (e_i)-\calP (0)\|^2.$$
When not stated otherwise, the basis $e_i$ will be the canonical basis of the ambient space.
\end{Def}
Note that coherence is always coherence with respect to the \emph{fixed} coordinate system of the sampling regime,
and this will be understood in what follows.
\begin{Rem}\label{Rem:cohamb}
Let $H\subseteq \KK^n$ be a $k$-flat. Then the coherence $\coh(H)$ does not depend on whether we consider $H$ as a $k$-flat in $\KK^n$, or as a $k$-flat in $\KK^m\supseteq \KK^n$ for $m\ge n$ (assuming the chosen basis of $\KK^m$ contains the basis of $\KK^n$). Moreover, if $H\subseteq \mathbb{R}^n$, the coherence of $H$ equals that of the complex closure of $H$. Therefore, while coherence depends on the choice of coordinate system, it is invariant under extensions of the coordinate system.
\begin{figure*}[htbp]
\centering
\subfigure[]{\includegraphics[width=0.2 \textwidth]{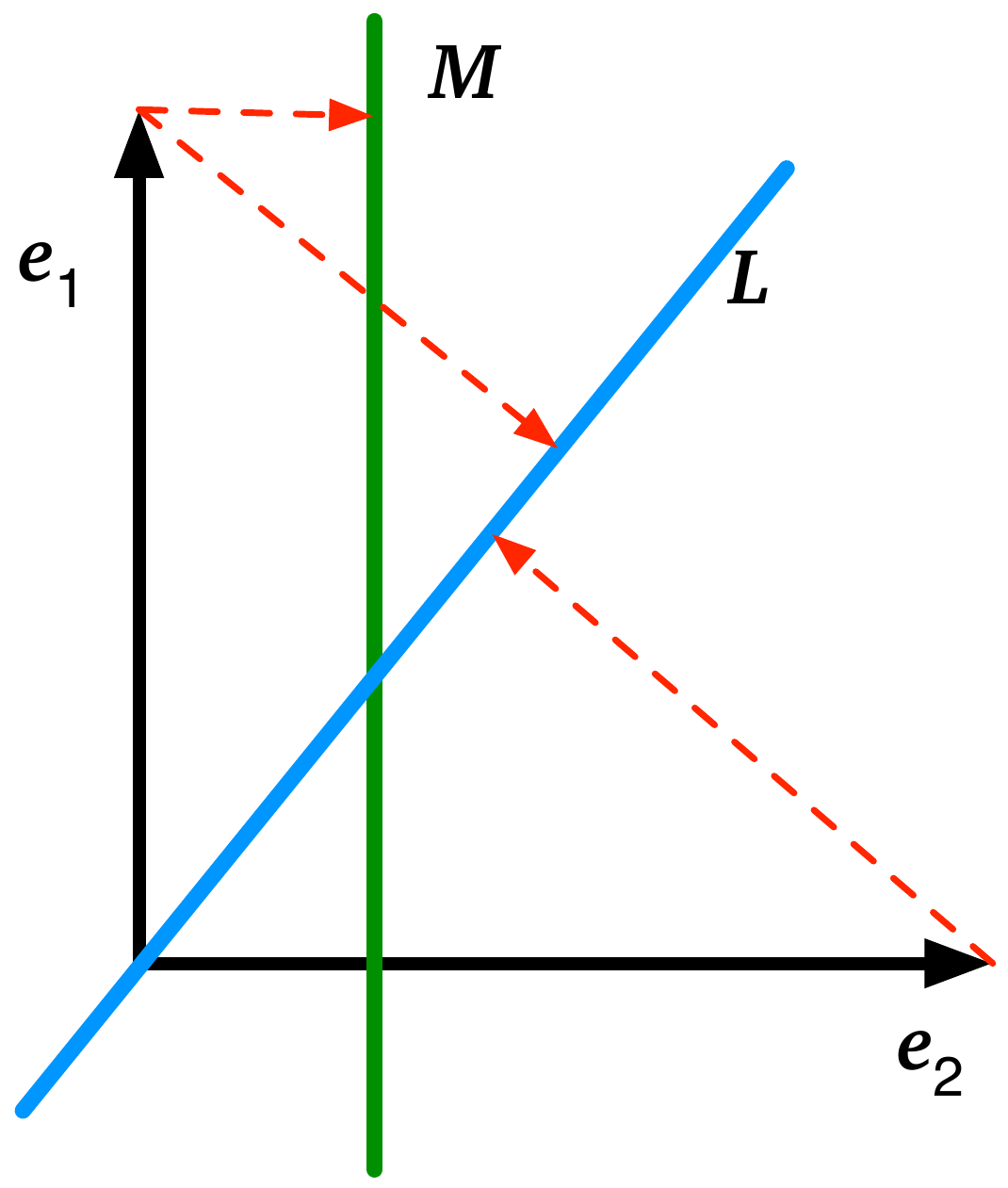}}
\hspace{1 in}
\subfigure[]{\includegraphics[width=0.2 \textwidth]{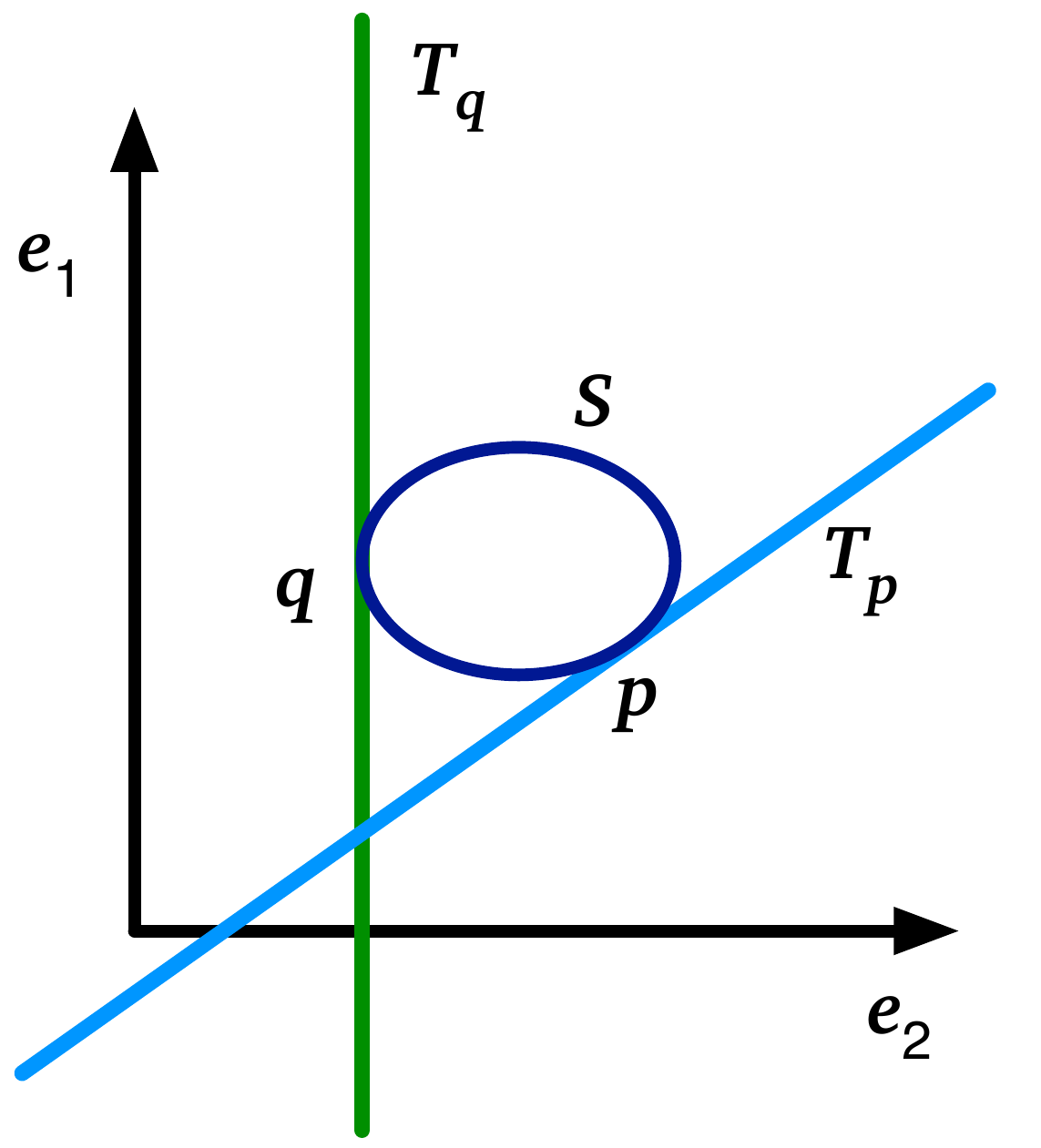}}
\caption{Schematic of coherent and incoherent spaces: (a) The projections of the coordinate
vectors onto the linear space $L$ are roughly the same size, so it has nearly minimal coherence.  The
flat $M$ is a translate of the span of $e_1$, giving it maximal coherence; observing the $e_2$
coordinate gives no information about any point in $M$.  (b) A variety $S$ in $\RR^2$ and the tangent
flats at two points.  The coherence of $S$ is close to minimal, as witnessed by the point $p$.}
\label{fig:coh}
\end{figure*}
\end{Rem}
A crucial property of the coherence is that it is bounded in both directions:
\begin{Prop}\label{Prop:cohbnds}
Let $H$ be a $k$-flat in $\KK^n.$ Then,
$$\frac{k}{n}\le\coh (H) \le 1,$$
and both bounds are achieved.
\end{Prop}
\begin{proof}
Without loss of generality, we can assume that $0\in H$ and therefore that $\calP$ is linear, since coherence, as defined in Definition~\ref{Def:coh}, is invariant under translation of $H$.
First we show the upper bound. For that, note that for an orthogonal projection operator $\calP:\KK^n\rightarrow \KK^n$ and any $x\in\KK^n$, one has $\|\calP (x)\|\le \|x\|.$ Thus, by definition,
$\coh (H) = \max_{1\le i\le n} \|\calP (e_i)\|^2\le \max_{1\le i\le n} \|e_i\|^2 =1.$
For strictness, take $H$ as the span of $e_1,\dots, e_k.$
Now we show the lower bound. We proceed by contradiction. Assume $\|P (e_i)\|^2 < \frac{k}{n}$ for all $i.$ This would imply
$ k = n\cdot\frac{k}{n} > \sum_{i=1}^n \|\calP (e_i)\|^2 = \|\calP \|_F^2 =k$
which is a contradiction, where in the last equality we used the fact that orthonormal projections onto a $k$-dimensional space have
Frobenius norm $k$.
When $n/k$ is an integer, the tightness of the lower bound follows from the example in section \ref{sec:intro.fixed}.
In general, it follows
from the existence of finite tight
frames\footnote{We cordially thank Andriy Bondarenko for pointing this out.} \cite{CasLeo06}.
\end{proof}

We extend the coherence to arbitrary manifolds by minimizing over tangent spaces; see Figure \ref{fig:coh} (b)
for an example.
\begin{Def}\label{Def:cohX}
Let $\varX\subseteq\KK^n$ be an (real or complex) irreducible analytic variety of dimension $d$ (affine or projective).
Let $x\in\varX$ a smooth point, and let $T_{\varX, x}$ be the tangent $d$-flat of $\varX$ at $x$. We define
$$\coh (x\in\varX) := \coh (T_{\varX, x}).$$
If it is clear from the context in which variety we consider $x$ to be contained, we also write $\coh (x)=\coh(x\in\varX).$ Furthermore, we define the coherence of $\varX$ to be
$$\coh (\varX) := \inf_{x\in\Sm(\varX)}\coh (x),$$
where $\Sm(\varX)$ denotes the set of smooth points (=the so-called smooth locus) of $\varX$.
\end{Def}

Remark~\ref{Rem:cohamb} again implies that the coherence $\coh (\varX)$ is invariant under the choice of ambient space and depends only on the coordinate system. Also, if $\varX$ is a $k$-flat, then the definitions of $\coh (\varX)$, given by Definitions~\ref{Def:coh} and \ref{Def:cohX} agree.
Therefore, we again obtain:

\begin{Prop}\label{Prop:cohbndsvar}
Let $\varX\subseteq\KK^n$ be an irreducible analytic variety. Then,
$\frac{1}{n}\dim\varX\le\coh (\varX) \le 1$, and both bounds are tight.
\end{Prop}
\begin{proof}
Let $d = \dim\varX$.  Irreducibility of $X$ implies that, at each smooth point $x\in \Sm(\varX)$, the tangent
space $T_{\varX, x}$ is a $d$-flat in $\KK^n$. Both bounds and their tightness then follow from Proposition \ref{Prop:cohbnds}.
\end{proof}

\begin{Def}\label{Def:maxinc}
An analytic variety $\varX\subseteq \KK^n$ is called {\it maximally incoherent} if
$\coh (\varX)=\frac{1}{n}\dim\varX.$
\end{Def}

\subsection{The Main Theorem}\label{Sec:coherence.mainthm}
With all concepts in place, we state our main result, which we recall from the introduction.
\cohpthm*
\begin{proof}
By the definition of coherence, for every $\delta > 0$, there exists an $x$ such that $\varX$ is smooth at $x$, and $\coh(x)\le (1+\delta)\coh (\varX).$ Now let $y = \Omega(x)$, we can assume by possible changing $x$ that $\Omega(\varX)$ is also smooth at $y$. Let $T_y,T_x$ be the respective tangent spaces at $y$ and $x$. Note that $y$ is a point-valued discrete random variable, and $T_y$ is a flat-valued random variable. By the equivalence of the statements (iv) and (v) in Lemma~\ref{Lem:injty}, it suffices to show that the operator
$$P = \srate^{-1}\theta\circ \diff\Omega - \id$$
is contractive, where $\theta$ is projection, from $T_y$ onto $T_x$, with probability at least $1-3n^{-\lambda}$ under the assumptions on $\srate$. Let $Z=\|P\|,$ and let $e_1,\dots, e_n$ be the orthonormal coordinate system for $\mathbb{C}^n$, and $\calP$ the projection onto $T_x$. Then the projection $\theta\circ \diff\Omega$ has, when we consider $T_x$ to be embedded into $\mathbb{C}^n$, the matrix representation
{\small
$$\sum_{i=1}^n\varepsilon_i\cdot \calP (e_i)\otimes \calP (e_i),$$
}
where $\varepsilon_i$ are independent Bernoulli random variables with probability $\srate$ for $1$ and $(1-\srate)$ for $0$. Thus, in matrix representation,
{\small
$$P=\sum_{i=1}^n\left(\frac{\varepsilon_i}{\srate}-1\right)\cdot \calP (e_i)\otimes \calP (e_i).$$
}
By Rudelson's Lemma~\ref{Lem:Rudelson}, it follows that
{\small
$$\mathbb{E} (Z) \le C \sqrt{\frac{\log n}{\srate}} \max_i \|\calP (e_i)\|$$
}
for an absolute constant $C$ provided the right hand side is smaller than $1$. The latter is true if and only if
$$\srate\ge C^{-2} \log n \max_i \|\calP (e_i)\|^2.$$
Now let $U$ be an open neighborhood of $x$ such that $\coh (z) < (1+\delta)\coh(\varX)$ for all $z\in U$. Then, one can write
{\small
$$Z=\sup_{y_1,y_2\in U'}\left\|\sum_{i=1}^n\left(\frac{\varepsilon_i}{\srate}-1\right)\cdot \langle y_1,\calP (e_i)\rangle \langle y_2, \calP (e_i)\rangle\right\|$$
}
with a countable subset $U'\subsetneq U.$ By construction of $U'$, one has
{\small
$$\left\|\left(\frac{\varepsilon_i}{\srate}-1\right)\cdot \langle y_1,\calP (e_i)\rangle \langle y_2, \calP (e_i)\rangle\right\|\le \srate^{-1}(1+\delta)\coh (\varX).$$
}
Applying Talagrand's Inequality in the form \citep[Theorem 9.1]{CR09}, one obtains
{\small
$$P(\|Z-\mathbb{E}(Z)\|>t)\le 3 \exp\left(-\frac{t}{KB}\log \left(1+\frac{t}{2}\right)\right)$$
}
with an absolute constant $K$ and $B= \srate^{-1}(1+\delta)\coh (\varX).$ Since $\delta$ was arbitrary, it follows that
{\small
$$P(\|Z-\mathbb{E}(Z)\|>t)< 3 \exp\left(-\frac{\srate \cdot t}{K \coh(\varX)}\log \left(1+\frac{t}{2}\right)\right).$$
}
Substituting $\srate= C\cdot \lambda'\cdot \coh (\varX)\cdot  \log n$, and proceeding as in the proof of Theorem 4.2 in \citep{CR09} (while changing absolute constants), one arrives at the statement.
\end{proof}
\begin{Rem}
That the manifold $\calX$ in the theorem needs to be algebraic is no major restriction,
since in the cases we are going to consider, the dependencies in $\calX$ will be algebraic,
or can be made algebraic by a canonical transform.  Moreover, Theorem \ref{Thm:cohp}
cannot be expected to hold for general analytic manifolds, since one might ``piece together''
pieces of manifolds with different identifiability characteristic.  For such an
object there is not global, prototypical generic behavior.
\end{Rem}

\begin{Rem}
By the bounds given in Proposition~\ref{Prop:cohbndsvar}, the best obtainable bound in Theorem~\ref{Thm:cohp} is
$\rho \ge C\cdot \lambda\cdot \dim(\varX)\cdot n^{-1} \log n,$ with $\lambda\ge 1$, in the case where $\varX$ is maximally incoherent.
\end{Rem}

\subsection{Coherence of subvarieties and secants}\label{Sec:coherence.subvars}
In the following section, we derive some further results how coherence behaves under restriction, and summation of signals, which will prove useful for computing or bouding coherence in specific examples.

\begin{Lem}\label{Lem:cohboundflat}
Let $H\subseteq \KK^n$ be a $k$-flat, let $\varX\subseteq H$ be a subvariety. Then,
$\coh (\varX)\le \coh(H).$
\end{Lem}
\begin{proof}
We first prove the statement for the case where $\varX$ is a flat; without loss of generality one can then assume that $0\in\varX$.
Let $\calP'$ be the unitary projection onto $X$, similarly $\calP$ the unitary projection onto $H$.
Since $X\subseteq H$, it holds that $\|\calP' x\|\le \|\calP x\|$ for any $x\in\KK^n$.
Thus, $\coh(\varX)\le\coh(H).$

The statement for the case where $\varX$ is an irreducible variety follows from the statement for vector spaces.
Namely, for $x\in \varX$, it implies $\coh(x\in\varX)\le\coh(H)$, since the tangent space of $\varX$ at $x$ is
contained in $H$. By taking the infimum, we obtain the statement.
\end{proof}

\begin{Lem}\label{Lem:sects}
Let $\varX,\varY\subseteq \KK^n$ be a analytic varieties, let
$\varX+\varY=\{\ x + y\;;\;x\in \varX, y\in \varY \}$
be the sum of $X$ and $Y$. Then,
$\coh (\varX)\le \coh(\varX + \varY).$
\end{Lem}
\begin{proof}
Denote $\varZ = \varX + \varY$, let $z\in\varZ$ be an arbitrary smooth point. By definition, there are smooth $x(z)\in\varX, y(z)\in \varY$ such that $z=x(z)+y(z)$. Let $T_z$ be the tangent space to $\varX+\varY$ at $z$, let $T_x$ be the tangent space of $\varX$ at $x(z)$. An elementary calculation shows $T_x\subseteq T_z$, thus $\coh(x(z))\le \coh(z)$ by Lemma~\ref{Lem:cohboundflat}. Since $z$ was arbitrary, we have
\(\coh(\varX)\le \inf_{z\in\Sm (\varZ)}\coh (x(z))\le \coh(\varZ) \).
\end{proof}

\begin{Rem}
In general, it is false that $\coh(\varX + \varY)\le \coh (\varX) + \coh (\varY).$ Consider for example $\varX = \spn ((1,1,1)^\top)$ and $\varY = \spn ((1,-1,-1)^\top)$.
\end{Rem}

\section{Coherence for Matrix Completion, Rigidity, and Kernels}\label{Sec:examples}
\subsection{Low-Rank Matrix Completion: the Determinantal Variety}\label{Sec:examples.lowrank}

In this section, we compute the coherence for completion of non-symmetric and symmetric bounded rank matrices, and then apply Theorem~\ref{Thm:cohp} to obtain boundary sampling rates for identifiability in matrix completion.

\begin{Def}
Denote by $\DetV (m\times n,r)$ the set of $(m\times n)$ matrices in $\KK$ of rank $r$ or less, and by $\DetVSym (n,r)$ the set of symmetric real resp.~Hermitian complex $(n\times n)$ matrices of rank $r$ or less, i.e.,
\begin{align*}
\DetV (m\times n,r) & = \left\{A\in\KK^{m\times n}\;;\;\rk A\le r \right\},\;\mbox{and}\\
\DetVSym (n,r) & = \left\{A\in\KK^{n\times n}\;;\;\rk A\le r, A^\dagger = A \right\}.
\end{align*}
Since the matrices in $\DetVSym (n,r)$ are symmetric resp.~Hermitian, we will consider it as canonically embedded in $\frac{1}{2}n(n+1)$-space.

$\DetV (m\times n,r)$ is called the {\it determinantal variety} of $(m\times n)$-matrices of rank (at most) $r$, $\DetVSym (n,r)$ the {\it determinantal variety} of symmetric $(n\times n)$-matrices of rank (at most) $r$.
\end{Def}

We first obtain the coherences of fixed matrices:
\begin{Prop}\label{Prop:rkmatinc}
Let $A\in\KK^{m\times n},$ let $H_m$ be the row span of $A$, and $H_n$ the column span of $A$.
Then, \(1-\coh \left(A\in \DetV (m\times n,r)\right) = (1-\coh(H_n))\cdot (1-\coh (H_m)) \) and,
if $A$ is symmetric/resp.~Hermitian, then \( 1- \coh \left(A\in \DetVSym (n,r)\right) = (1-\coh(H_n))^2\).
\end{Prop}
\begin{proof}
The calculation leading to \cite[Equation 4.9]{CR09} shows in both cases that $\coh(A)=
\coh(H_n)+\coh(H_m)-\coh(H_n)\coh(H_m)$, from which the statement follows.
\end{proof}

\begin{Prop}\label{Prop:cohdet}
\begin{align*}
\coh (\DetV (m\times n,r))&=\frac{r}{mn}\cdot (m+n-r)\\
\coh(\DetVSym (n,r))&= \frac{r}{n^2}\cdot (2n-r).
\end{align*}
In particular, $\DetV (m\times n,r)$ is maximally incoherent, whereas $\DetVSym (n,r)$ is not.
\end{Prop}
\begin{proof}
We recall the fact that any pair of linear $r$-flats $H_n$ and $H_m$ in $m$-resp.~$n$-space, there exists an $A\in \DetV (m\times n,r)$ such that the row span of $A$ is exactly $H_m$, and the column span of $A$ is exactly $H_n$. Similarly, there is $B\in \DetVSym (n,r)$ such that row and column span of $B$ are equal to $H_n$.
By Proposition~\ref{Prop:cohbnds}, there exist $r$-flats $H_n$ and $H_m$ with $\coh (H_n)=k/n$ and $\coh (H_m) = k/m$. Therefore, by Proposition~\ref{Prop:rkmatinc}, there is $A\in\DetV (m\times n,r)$ with $\coh (A)=\frac{r}{mn}\cdot (m+n-r)$, so $\coh (\DetV (m\times n,r))=\coh(A)$ follows from the lower bound in Proposition~\ref{Prop:cohbndsvar}.
For the equality $\coh(\DetVSym (n,r))= \frac{r}{n^2}\cdot (2n-r)$, it suffices to show $\coh (\DetV (n\times n,r)) = \coh(\DetVSym (n,r)).$ The inequality $\coh (\DetV (n\times n,r)) \le \coh(\DetVSym (n,r))$ follows from Proposition~\ref{Prop:rkmatinc} by considering $\DetV (n\times n,r)\subseteq \DetVSym (n,r)$. For the converse, let $B\in \DetV (n\times n,r).$ It suffices to show that there is $M\in \DetVSym (n,r)$ with $\coh(M)\le \coh(B)$. Let $H_1,H_2$ be row and column span of $B$, such that $\coh(H_1)\le\coh(H_2)$. Choosing an $M$ with column (and thus also row) span $H_1$ yields, by Proposition~\ref{Prop:rkmatinc}, an $M$ with $\coh(M)\le \coh(A)$.
\end{proof}

From our main Theorem~\ref{Thm:cohp}, we obtain the following corollary for low-rank matrices:
\lowrankthm*
\begin{proof}
Combine Theorem~\ref{Thm:cohp} with the explicit formula for the coherence in Proposition~\ref{Prop:cohdet}.
\end{proof}

\subsection{Distance Matrix Completion: the Cayley-Menger Variety}\label{Sec:examples.rigidity}

In this section, we will bound the coherence of the Cayley-Menger variety, i.e., the set of Euclidean distance matrices, by relating it to symmetric low-rank matrices. We first introduce notation for the set of signals:

\begin{Def}
Assume $r\le m\le n.$ We will denote by $\CayMen (n,r)$ the set of $(n\times n)$ real Euclidean distance matrices of points in $r$-space, i.e., %
\begin{align*}
\CayMen (n,r) = \{&D\in\KK^{n\times n}\;;\;D_{ij} = (x_i-x_j)^\top (x_i-x_j)\\
&\mbox{for some}\;x_1,\dots, x_n\in\KK^r\}.%
\end{align*}
Since the the elements of $\CayMen (n,r)$ are symmetric, and have zero diagonals, we will consider $\CayMen (n,r)$ as canonically embedded in $\binom{n}{2}$-space.

$\CayMen (n,r)$ is called the {\it Cayley-Menger variety} of $n$ points in $r$-space.
\end{Def}

We will now continue with introducing maps related to the above sets:
\begin{Def}
We define canonical surjections
\begin{align*}
\CayMenSurj:& \left(\KK^{r}\right)^n\rightarrow \CayMen (n,r);\\
&(x_1,\dots, x_n)\mapsto D\;\mbox{s.t.}\;D_{ij} = (x_i-x_j)^\top (x_i-x_j),\\
\DetVSymSurj:& \left(\KK^{r}\right)^n\rightarrow \DetVSym (n,r);\\
&(x_1,\dots, x_n)\mapsto A\;\mbox{s.t.}\;A_{ij} = x_i^\top x_j .
\end{align*}
Note that $\CayMenSurj, %
\DetVSymSurj$ depend on $r$ and $n$, but are not explicitly written as parameters in order to keep notation simple. Which map is referred to will be clear from the format of the argument.
\end{Def}

We now define a ``normalized version'' of $\DetVSym (n,r)$:
\begin{Def}
Denote by $\mathbb{S}^r =\{x\in\KK^{r+1}\;;\; x^\top x=1\}$. Then, define
\(
\DetVSymN (n,r) := \DetVSymSurj\left(\left(\mathbb{S}^{r}\right)^n\right).
\)
Since $\DetVSymN (n,r)$ contains only symmetric matrices with diagonal entries one, we will consider it as a subset of $\binom{n}{2}$-space.
\end{Def}

\begin{Rem}
The maps $\CayMenSurj, \DetVSymSurj$ are algebraic maps, and the sets\\ $\CayMen (n,r), %
\DetV (m\times n,r), \DetVSym (n,r), \DetVSymN (n,r)$ are irreducible algebraic varieties\footnote{irreducibility for $\CayMen (n,r), %
\DetVSym (n,r), \DetVSymN (n,r)$ follows from irreducibility of the respective ranges of the complex closure of $\CayMenSurj, \DetVSymSurj$ and surjectivity, irreducibility of $\DetV (m\times n,r)$ can be shown in a similar way; note that the real maps are in general not surjective}.
\end{Rem}

\begin{Lem}\label{Lem:DetSurjRC}
For arbitrary $n,r$, one has $\coh \left(\DetVSym (n,r)\right)=\coh \left(\DetVSymSurj\left(\KK^{r}\right)^n\right).$
\end{Lem}
\begin{proof}
If $\KK=\mathbb{C}$, then $\DetVSymSurj$ is surjective, so the statement follows. If $\KK=\mathbb{R}$, note that the coherence of a general matrix does not depend on the variety it is considered in, since $\dim \DetVSym (n,r) =  \dim \DetVSymSurj\left(\KK^{r}\right)^n$. Take $M\in\DetVSym (n,r)$. Then, take any matrix $A\in\mathbb{R}^{n\times r}$ whose rows are a basis for the row span of $M$. Then, $A A^\top\in\DetVSymSurj \left(\mathbb{R}^{r}\right),$ and by Proposition~\ref{Prop:rkmatinc}, $\coh(M) = \coh(A A^\top).$ The statement follows from this.
\end{proof}
The dimensions of the above varieties are classically known:
\begin{Prop}
One has $\dim \CayMen (n,r)  %
= \dim \DetVSymN (n,r+1)= r\cdot n - \binom{r+1}{2},$
and the dimensions are the same for the complex closures.
\end{Prop}
Central in the proof will be the following map:
\begin{Def}
For $h\in\KK,$ we will denote by
$$\SphPr: \KK^{r} \rightarrow \mathbb{S}^{r}\;;\; x \rightarrow \frac{1}{\sqrt{x^\top x+h^2}}(x,h)$$
the map which considers a point $\KK^{r}$ as a point in the hyperplane $\{(x,h)\;;\;x\in\mathbb{R}^r\}\subseteq \mathbb{R}^{r+1}$ and projects it onto $\mathbb{S}^{r}$.
(if $\KK=\CC,$ we fix any branch of the square root)
\end{Def}

\begin{Prop}\label{Prop:CayDetcoh}
For any $n,r,$ it holds one has
$\coh (\CayMen (n,r)) \le \coh (\DetVSymN (n,r+1))$.%
\end{Prop}
\begin{proof}
Lemma~\ref{Lem:tancaydet} implies that
$\coh \left(\CayMenSurj\left(\left(\KK^{r}\right)^n\right)\right)\le \coh \left(\DetVSymSurj\left(\left(\KK^{r}\right)^n\right)\right),$
the claim then follows from $\CayMenSurj\left(\left(\KK^{r}\right)^n\right)\subseteq \CayMen (n,r)$ and Lemma~\ref{Lem:DetSurjRC}.
\end{proof}
We can bound the coherence of $\DetVSymN(n,r)$ as follows:
\begin{Prop}\label{Prop:DetVSymNCohbound}
There is a global constant $C$, such that $\coh(\DetVSymN (n,r)) \le C \frac{r}{n}$.
\end{Prop}
\begin{proof}
It follows from~\cite[Lemma 2.2]{CR09} that for any fixed set of singular values there exists a matrix $M\in \DetV (n\times n,r)$ with $\coh(M)\le Cn^{-1}r$ such that $M$ has these singular values. By taking the singular values of $M$ to be all one, and replacing $M$ with a symmetric matrix $M'$ having the same row or column span as $M$, as in the proof of Proposition~\ref{Prop:cohdet}, we see by Proposition~\ref{Prop:rkmatinc} that $\coh\left( M'\in\DetVSymN (n,r)\right)\le \coh(M).$
\end{proof}
Our stated bounds on the number of samples required for distance matrix reconstruction then follow
from the following lemma:
\begin{Lem}\label{Lem:tancaydet}
Let $x_1,\dots, x_n\in\mathbb{R}^{r}.$
Let $D = \CayMenSurj (x_1,\dots, x_n)$ and $A=\DetVSymSurj (\SphPr (x_1),\dots, \SphPr(x_n))$, let $T_D, T_A$ the respective tangent flats.  Then, for $h\rightarrow\infty$, we have convergence $T_A\rightarrow T_D,$ where we consider the tangent flats as points on the real Grassmann manifold of $\left(r\cdot n - \binom{r+1}{2}\right)$-flats in $\binom{n+1}{2}$-space.
\end{Lem}
\begin{proof}
Note that
\begin{align*}
D_{ij} &= x_i^\top x_i - 2x_i^\top x_j + x_j^\top x_j,\\
A_{ij} &= \frac{x_i^\top x_j + h^2}{\sqrt{x_i^\top x_i + h^2}\sqrt{x_j^\top x_j + h^2}},
\end{align*}
An explicit calculation shows:
\begin{align*}
\left(\frac{\partial D}{\partial x_k}\right)_{ij} & = 2(\delta_{ki}+\delta_{kj})(x_i-x_j)\\
\left(\frac{\partial A}{\partial x_k}\right)_{ij} & = %
-(\delta_{ki}+\delta_{kj})
\frac{x_i(x_i^\top x_j + h^2)\sqrt{\frac{x_j^\top x_j +h^2}{x_i^\top x_i +h^2}} - x_j\sqrt{x_i^\top x_i + h^2}\sqrt{x_j^\top x_j + h^2}}{\left(x_i^\top x_i + h^2\right)\left(x_j^\top x_j + h^2\right)},
\end{align*}
where $\delta_{ij}$ is the usual Kronecker delta. Thus,
\begin{align*}
\lim_{h\rightarrow \infty} h^2 \left(\frac{\partial A}{\partial x_k}\right)_{ij} & = -\frac{1}{2} \left(\frac{\partial D}{\partial x_k}\right)_{ij}
\end{align*}
which implies that both $T_A$ %
converges to $T_D$ in the Grassmann manifold when taking the limit $h\rightarrow \infty$; the statement directly follows.
\end{proof}

\rigiditythm*

\begin{proof}
This follows from Theorem~\ref{Thm:cohp} and the coherence bounds from
Propositions~\ref{Prop:CayDetcoh} and~\ref{Prop:DetVSymNCohbound}.
\end{proof}
\subsection{Kernels}\label{Sec:examples.kernels}
Our framework can also be applied to analyze kernel functions via their coherence; namely, coherence can be interpreted as the average contribution one entry of the kernel matrix makes to characterize the whole of the data. While the set of kernel matrices is in general not algebraic anymore, it is analytic, and can be related to the examples above, yielding the following result:
\begin{Thm}\label{Thm:kernels}
Let $k:\RR^d\times \RR^d\rightarrow \RR$ be a polynomial kernel or an RBF kernel, let $K$ be an $(n\times n)$ symmetric kernel matrix in $k$.
Then, there is a global constant $C$, such that if each entry of $K$ is observed independently with probability
$$\srate \ge C\cdot \lambda \cdot d/n \cdot \log n, \quad\mbox{with}\;\lambda\ge 1,$$
then $K$ is determined by the observations with probability at least $1-3n^{-\lambda}$.
\end{Thm}
\begin{proof}
This follows from Theorems~\ref{Thm:generic-low-rank} and~\ref{Thm:rigidity}, and the fact the entries of $K$ are finite (degree) functions over either a rank-$d$- or a distance matrix.
\end{proof}
Theorem~\ref{Thm:kernels} means that while kernel matrices are not necessarily algebraic, they also exhibit sampling bounds with a coherence-equivalent of $d/n$.

\section{Conclusion}
We expect that the framework presented here will serve as the basis for
investigations into a broader set of applications than just the examples
here.

Also, we would expect an investigation of Theorem \ref{Thm:cohp} for different sampling scenarios to be very
interesting.

Namely, one can ask in which cases the $\log n$ term can be removed, in dependence of the particular sampling distribution, or the signal space - keeping in mind that the coupon collector's lower bound is not compulsory in every scenario, and that various results exist, which assert, under different sampling assumptions or other kinds of sparsity assumptions, bounds that are linear in $n$.

Any result along these lines would potentially allow us to address
the question of the required sampling rates needed for reconstruction
of only a linear-size fraction of the coordinates of the signal, which
is enough for many practical scenarios.

\newpage

\bibliographystyle{plainnat}

\newpage
\appendix

\section{Finiteness of Random Projections}
The theorem, which will be proved in this section and which is probably folklore, states that for a {\it general} system of coordinates, a number of $\dim (\varX)$ observation is sufficient for identifiability.

\begin{Thm}\label{Thm:randprojApp}
Let $\varX\subseteq\mathbb{K}^n$ be an algebraic variety or a compact analytic variety, let $\Omega:\mathbb{K}^n\rightarrow \mathbb{K}^m$ a generic linear map. Let $x\in \varX$ be a smooth point. Then, $\varX\cap\Omega^{-1}(\Omega(x))$ is finite if and only if $k\ge \dim (\varX)$, and $\varX\cap\Omega^{-1}(\Omega(x))=\{x\}$ if $m > \dim (\varX ).$
\end{Thm}
\begin{proof}
The theorem follows from the the more general height-theorem-like statement that
\begin{align*}
&\codim \left(\varX\cap H\right)\\
&= \codim (\varX)+\codim(H) = \codim (\varX) + n - k,
\end{align*}
where $H$ is a generic $k$-flat. Then, the first statement about generic finiteness follows by taking a generic $y\in\Omega(\varX)$ and observing that $\Omega^{-1}(y)=H\cap \varX$ where $H$ is generic if $k\le \dim(\varX)$. That implies in particular that if $k=\dim(\varX)$, then the fiber $\Omega^{-1}(\Omega(x))$ for a generic $x\in\varX$ consists of finitely many points, which can be separated by an additional generic projection, thus the statement follows.
\end{proof}

Theorem~\ref{Thm:randprojApp} can be interpreted in two ways. On one hand, it means that any point on $\varX$ can be reconstructed from exactly $\dim(\varX)$ random linear projections. On the other hand, it means that if the chosen coordinate system in which $\varX$ lives is random, then $\dim(\varX)$ measurements suffice for (finite) identifiability of the map - no more structural information is needed. In view of Theorem~\ref{Thm:cohp}, this implies that the log-factor and the probabilistic phenomena in identifiability occur when the chosen coordinate system is degenerate with respect to the variety $\varX$ in the sense that it is intrinsically aligned.

\section{Analytic Reconstruction Bounds and Concentration Inequalities}\label{Sec:mainproof}

This appendix collects some analytic criteria and bounds which are used in the proof of Theorem~\ref{Thm:cohp}. The first lemma relates local injectivity to generic finiteness and contractivity of a linear map. It is related to \citep[Corollary~4.3 ]{CR09}.

\begin{Lem}\label{Lem:injty}
Let $\varphi:\varX\rightarrow\varY$ be a surjective map of complex algebraic varieties, let $x\in \varX$, and $y=\varphi(x)$ be smooth points of $\varX$ resp.~$\varY$. Let
$$\diff \varphi: T_x\varX\rightarrow T_y\varY $$
be the induced map of tangent spaces\footnote{$T_x\varX$ is the tangent plane of $\varX$ at $x$, which is identified with a vector space of formal differentials where $x$ is interpreted at $0$. Similarly, $T_y\varY$ is identified with the formal differentials around $y$. The linear map $\diff\varphi$ is induced by considering $\varphi(x+\diff v) = y + \diff v'$ and setting $\diff \varphi (\diff v) = \diff v';$ one checks that this is a linear map since $x,y$ are smooth. Furthermore, $T_x\varX$ and $T_y\varY$ can be endowed with the Euclidean norm and scalar product it inherits from the tangent planes. Thus, $\diff\varphi$ is also a linear map of normed vector spaces which is always bounded and continuous, but not necessarily proper. }. Then, the following are equivalent:
\begin{description}
\item[(i)] There is an complex open neighborhood $U\ni x$ such that the restriction $\varphi:U\rightarrow \varphi(U)$ is bijective.
\item[(ii)] $\diff \varphi$ is bijective.
\item[(iii)] There exists an invertible linear map $\theta:T_y\varY\rightarrow T_x\varX.$
\item[(iv)] There exists a linear map $\theta:T_y\varY\rightarrow T_x\varX$ such that the linear map
$$\theta\circ \diff\varphi - \id,$$
where $\id$ is the identity operator, is contractive\footnote{A linear operator $\calA$ is contractive if $\|\calA(x)\|<1$ for all $x$ with $\|x\|< 1$.}.
\end{description}

If moreover $\varX$ is irreducible, then the following is also equivalent:
\begin{description}
\item[(v)] $\varphi^{-1}(y)$ is finite for generic $y\in\varY$.
\end{description}

\end{Lem}
\begin{proof}
(ii) is equivalent to the fact that the matrix representing $\diff\varphi$ is an invertible matrix. Thus, by the properties of the matrix inverse, (ii) is equivalent to (iii), and (ii) is equivalent to (i) by the constant rank theorem (e.g., 9.6 in \citet{RudinAnalysis}).\\

By the upper semicontinuity theorem (I.8, Corollary 3 in~\citet{Mumford}), (i) is equivalent to (v) in the special case that $\varX$ is irreducible.\\

(ii)$\Rightarrow$ (iv): Since $\diff\varphi$ is bijective, there exists a linear inverse $\theta:T_y\varY\rightarrow T_x\varX$ such that $\theta\circ\diff\varphi = \id.$ Thus
$$\theta\circ \diff\varphi - \id = 0$$
which is by definition a contractive linear map.\\

(iv)$\Rightarrow$ (iii): We proceed by contradiction. Assume that no linear map $\theta:T_y\varY\rightarrow T_x\varX$ is invertible.
Since $\varphi$ is surjective, $\diff\varphi$ also is, which implies that for each $\theta$, the linear map $\theta\circ \diff\varphi$ is rank deficient. Thus, for every $\theta$, there exists a non-zero $\alpha\in\Ker \theta.$ By linearity and surjectivity of $\diff\Omega$, there exists a non-zero $\beta\in T_x\varX$ with $\diff\Omega(\beta)=\alpha.$ Without loss of generality we can assume that $\|\beta\|=1$, else we multiply $\alpha$ and $\beta$ by the same constant factor. By construction,
$$ \left\|[\theta\circ \diff\varphi - \id](\beta)\right\|= \|\theta (\alpha) - \beta\|=\|\beta\|=1, $$
so $\theta$ cannot be contractive. Since $\theta$ was arbitrary, this proves that (iv) cannot hold if (iii) does not hold, which is equivalent to the claim.
\end{proof}

The second lemma is a consequence of Rudelson's Lemma, see \citet{Rudelson1999}, for Bernoulli samples.

\begin{Lem}\label{Lem:Rudelson}
Let $y_1,\dots, y_M$ be vectors in $\mathbb{R}^n,$ let $\varepsilon_1,\dots, \varepsilon_M$ be i.i.d.~Bernoulli variables, taking value $1$ with probability $p$ and $0$ with probability $(1-p)$. Then,
$$\mathbb{E} \left(\left\| 1-\sum_{i=1}^M\left(\frac{\varepsilon_i}{p}\right)y_i\otimes y_i\right\|\right)\le C \sqrt{\frac{\log n}{p}} \max_{1\le i\le M} \|y_i\|$$
with an absolute constant $C$, provided the right hand side is $1$ or smaller.
\end{Lem}
\begin{proof}
The statement is exactly Theorem~3.1 in \citet{CandesRomberg}, up to a renaming of variables, the proof can also be found there. It can also be directly obtained from Rudelson's original formulation in \citet{Rudelson1999} by substituting $\frac{\varepsilon_i}{\sqrt{p}}y_i$ in the above formulation for $y_i$ in Rudelson's formulation and upper bounding the right hand side in Rudelson's estimate.
\end{proof}

\end{document}